\documentclass[11pt]{article}
\usepackage{amsmath,amsthm,amssymb}
\usepackage{textcomp}
\usepackage{mathtools,bbm,xspace}
\usepackage{mathrsfs}
\usepackage{nicefrac}
\usepackage[numbers]{natbib}
\usepackage[usenames,dvipsnames]{xcolor}
\usepackage{ifxetex}
\usepackage{dirtytalk}
\usepackage{soul}

\usepackage[pagebackref]{hyperref}
\hypersetup{
colorlinks=true,
urlcolor=blue,
linkcolor=blue,
citecolor=OliveGreen,
}

\usepackage[capitalise,nameinlink]{cleveref}
\crefname{lemma}{Lemma}{Lemmas}
\crefname{fact}{Fact}{Facts}
\crefname{theorem}{Theorem}{Theorems}
\crefname{corollary}{Corollary}{Corollaries}
\crefname{conjecture}{Conjecture}{Conjectures}
\crefname{claim}{Claim}{Claims}
\crefname{example}{Example}{Examples}
\crefname{problem}{Problem}{Problems}
\crefname{setting}{Setting}{Settings}
\crefname{definition}{Definition}{Definitions}
\crefname{assumption}{Assumption}{Assumptions}
\crefname{subsection}{Subsection}{Subsections}
\crefname{section}{Section}{Sections}

\usepackage{tcolorbox}
\usepackage{enumerate}
\usepackage{csquotes}
\usepackage{algorithm}
\usepackage{algorithmic}
\usepackage{tabularx}

\usepackage{todonotes}

\usepackage{bm}
\allowdisplaybreaks

\newtheorem{theorem}{Theorem}[section]
\newtheorem*{theorem*}{Theorem}

\newtheorem*{proposition*}{Proposition}
\newtheorem{lemma}[theorem]{Lemma}
\newtheorem*{lemma*}{Lemma}

\newtheorem*{conjecture*}{Conjecture}

\newtheorem*{fact*}{Fact}

\newtheorem*{exercise*}{Exercise}

\newtheorem*{hypothesis*}{Hypothesis}
\newtheorem{conjecture}[theorem]{Conjecture}

\usepackage{thm-restate}

\theoremstyle{definition}

\newtheorem{exercise-easy}[theorem]{Exercise}
\newtheorem{exercise-med}[theorem]{Exercise}
\newtheorem{exercise-hard}[theorem]{Exercise$^\star$}

\newtheorem*{claim*}{Claim}

\newtheorem*{remark*}{Remark}

\newtheorem*{observation*}{Observation}

\newcommand{\Sref}[1]{\hyperref[#1]{\S\ref*{#1}}}

\usepackage{parskip}
\usepackage[
letterpaper,
top=1in,
bottom=1in,
left=1in,
right=1in]{geometry}

\DeclarePairedDelimiterX{\norm}[1]{\lVert}{\rVert}{#1}

\DeclarePairedDelimiterX{\abs}[1]{\lvert}{\rvert}{#1}
\DeclarePairedDelimiterX{\inp}[2]{\langle}{\rangle}{#1, #2}

\DeclarePairedDelimiterX{\infdivx}[2]{(}{)}{%
  #1\;\delimsize\|\;#2%
}

\DeclareMathOperator*{\E}{\mathbb{E}}
\let\Pr\relax
\DeclareMathOperator*{\Pr}{\mathbf{Pr}}

\newcommand{\mc}[1]{\mathcal{#1}}

\newcommand{\err}[2]{\operatorname{err}_{#2}\left( #1 \right)}
\newcommand{\erm}[1]{\widehat{f}_{S_{#1}}}
\newcommand{\bad}{\widehat{f}}
\newcommand{\mega}{C}
\newcommand{\round}{\lfloor d \rfloor}
\newcommand{\maj}{\mathrm{Maj}}
\newcommand{\Log}{\operatorname{Log}}

\newcommand{\ins}{\mathcal{X}}

\newcommand{\fc}{\mathcal{F}}

\def\com{1}
\newcommand{\ishaq}[1]{
    \if\com1
        \todo[inline,color=blue!30]{\small\textbf{Ishaq:} #1}
    \else
    \fi
}
\newcommand{\kasper}[1]
{
    \if\com1
        \todo[inline,color=green!30]{\small\textbf{Kasper:} #1}
    \else
    \fi
}
\newcommand{\mikael}[1]
{
    \if\com1
        \todo[inline,color=red!30]{\small\textbf{Mikael:} #1}
    \else
    \fi
}

\newcommand{\nikita}[1]
{
    \if\com1
        \todo[inline,color=yellow!30]{\small\textbf{Nikita:} #1}
    \else
    \fi
}

\crefformat{equation}{(#2#1#3)}
\usepackage{graphicx} 
\title{Majority-of-Three: The Simplest Optimal Learner?}
 \author{Ishaq Aden-Ali\thanks{Department of EECS, UC Berkeley. Email: {adenali@berkeley.edu}} \and Mikael Møller Høgsgaard\thanks{Computer Science Department, Aarhus University. Email: {\{hogsgaard, larsen\}@cs.au.dk}} \and Kasper Green Larsen\footnotemark[2] \and Nikita Zhivotovskiy\thanks{Department of Statistics, UC Berkeley. Email: zhivotovskiy@berkeley.edu}}
\date{}
\begin{document}
\sloppy
\maketitle
\begin{abstract}
Developing an optimal PAC learning algorithm in the realizable setting, where empirical risk minimization (ERM) is suboptimal, was a major open problem in learning theory for decades.
The problem was finally resolved by Hanneke a few years ago.
Unfortunately, Hanneke's algorithm is quite complex as it returns the majority vote of many ERM classifiers that are trained on carefully selected subsets of the data.
It is thus a natural goal to determine the simplest algorithm that is optimal.
In this work we study the arguably simplest algorithm that could be optimal: 
returning the majority vote of three ERM classifiers.
We show that this algorithm achieves the optimal in-expectation bound on its error which is provably unattainable by a single ERM classifier. 
Furthermore, we prove a near-optimal high-probability bound on this algorithm's error. 
We conjecture that a better analysis will prove that this algorithm is in fact optimal in the high-probability regime.

\end{abstract}
\section{Introduction}
In the setting of realizable Probably Approximately Correct (PAC) learning~\cite{valiant1984theory, vapnik1964class, vapnik74theory}, the goal is to learn or approximate an unknown target function $f^\star \in \{0,1\}^\ins$ from a labelled training sample $(S,f^\star(S))=((X_1,f^\star(X_1)),\ldots,(X_n,f^\star(X_n)))$, where the $X_i$'s are i.i.d.\ samples from an unknown distribution $P$ over an instance space $\ins$. In the realizable setting, we are furthermore promised that $f^\star$ belongs to a known function class $\fc \subseteq \{0,1\}^\ins$ of Vapnik-Chervonenkis (VC) dimension $d$. 

Given a labelled training sample $(S,f^\star(S))$, a learning algorithm produces a function $\widehat{f}_S \in \{0,1\}^\ins$ with the goal of minimizing the probability of mispredicting the label of a new sample from $P$, where we denote this error by $\err{\widehat{f}_S}{P} := \Pr_{X \sim P}[\widehat{f}_S(X) \neq f^\star(X)]$. The simplest reasonable learning algorithm, known as \emph{empirical risk minimization} (ERM), simply reports an arbitrary function $\widehat{f}_S\in \fc$ that is consistent with $f^\star$ on the training data, i.e.  $\widehat{f}_S(X_i)=f^\star(X_i)$ for all $i=1,\dots,n$. Classic work by Blumer et al. \cite{blumer1989learnability} (the same bound also essentially follows from the earlier works \cite{vapnik1968algorithms, vapnik1971uniform}) shows that for any $\delta>0$, it holds with probability $1-\delta$ over $S$ that any $\widehat{f}_S  \in \mc{F}$ consistent with $f^\star$ on $S$ has 
\begin{equation}\label{eq:erm_error}
\err{\widehat{f}_S}{P} = O\left( \frac{d}{n}\log\left(\frac{n}{d}\right) + \frac{1}{n}\log\left(\frac{1}{\delta}\right)\right).
\end{equation}

On the lower bound side, there exists an instance space $\mathcal{X}$ and function class $\mathcal{F}$ such that for a certain ERM algorithm, there is a target function $f^\star\in\mathcal{F}$ and hard distribution $P$ for which \cref{eq:erm_error} is tight~\cite{haussler1994predicting,auer2007new,simon2015almost,hanneke2016refined}. 
Learning algorithms that always output a function in $\mathcal{F}$ are referred to as \emph{proper} learning algorithms. 
Generally, it is known that not only ERM, but all proper learners fail to achieve the optimal error bound in the PAC learning framework. See the corresponding lower bounds in \cite{bousquet2020proper}.

For \emph{improper} learning algorithms --- algorithms that are allowed to output an arbitrary function $\widehat{f}_S\in \{0,1\}^\ins$ --- known lower bounds on the error only imply that we must have \begin{equation}\label{eq:improper_lowerbound}
    \err{\widehat{f}_S}{P} = \Omega\left( \frac{d}{n} + \frac{1}{n}\log\left(\frac{1}{\delta}\right)\right).
\end{equation}

Developing an algorithm with a matching error upper bound, or strengthening the lower bound, was a major open problem for decades.
This was finally resolved in 2015 when Hanneke~\cite{hanneke2016optimal}, building on the work of Simon~\cite{simon2015almost}, proposed the first optimal algorithm with an error upper bound matching \cref{eq:improper_lowerbound}, leading to the optimal error bound
\begin{equation}
    \label{eq:optimal_bound}
    \Theta\left( \frac{d}{n} + \frac{1}{n}\log\left(\frac{1}{\delta}\right)\right).
\end{equation}
Hanneke's algorithm is based on constructing a large number ($\approx n^{0.79}$) of sub-samples $S_1,S_2,\dots \subseteq S$ of the training data.
This algorithm then runs ERM on each $(S_i,f^\star(S_i))$ to obtain functions $\widehat{f}_{S_1},\widehat{f}_{S_2},\dots$ and finally combines them via a majority vote. 
The sub-samples $S_i$ are constructed to have a carefully designed overlapping structure, and an intricate inductive argument exploiting this structure is then used to argue optimality. Recent work by Larsen~\cite{larsen2023baggingCOLT} shows that the carefully designed overlap structure may instead be replaced by the significantly simpler strategy of sampling each $S_i$ as $\Theta(n)$ samples with replacement from $S$. This algorithm is precisely the classic heuristic known as Bagging, or bootstrap aggregation, due to Breiman~\cite{Breiman2004BaggingP}. Furthermore, the proof shows that a mere $O(\log(n/\delta))$ sub-samples suffice for an optimal sample complexity. The proof is however even more involved than Hanneke's and uses his analysis at its core.

Another line of work studied an alternative learning algorithm, the one-inclusion graph algorithm of Haussler, Littlestone, and Warmuth~\cite{haussler1994predicting} that returns a function $\widehat{f}_{\mathrm{OIG}}$.
This work also introduces the \emph{prediction model} of learning, which focuses on achieving bounds on the \emph{expected error} rather than \emph{high probability} bounds on the error. The one-inclusion graph algorithm was initially shown to have an expected error of
\begin{equation}\label{eq:OIG_risk}
\E_{S\sim P^{n}} \left[\err{\widehat{f}_{\mathrm{OIG}}}{P}\right] \le \frac{d}{n+1},
\end{equation}
which was later proven to be optimal within this prediction model \cite{li2001one}.
Because of the tightness of the in-expectation bound \cref{eq:OIG_risk}, Warmuth conjectured \cite{warmuth2004optimal} that the one-inclusion graph algorithm achieves an error upper bound matching the general lower bound \cref{eq:improper_lowerbound} in the high probability regime. 

Recent work by Aden-Ali, Cherapanamjeri, Shetty, and Zhivotovskiy~\cite{Aden-AliCOLT23} unfortunately refutes this conjecture.
Concretely, they show that for any $d \in \mathbb{N}$, sample size $n \ge d$ and confidence parameter $\delta \geq c d/n$, there exists a function class $\mc{F} \subseteq \{0,1\}^{\mc{X}}$ with VC dimension $d$ and a hard distribution $P$ such that a certain implementation of the one-inclusion graph algorithm has, with probability at least $\delta$,
\[
\err{\widehat{f}_{\mathrm{OIG}}}{P} = \Omega\left(\frac{d}{\delta n}\right).
\]
This result essentially says that, in general, the one-inclusion graph algorithm's high-probability guarantee cannot be better than applying Markov's inequality to the in-expectation guarantee in \cref{eq:OIG_risk}.
In recent work also by Aden-Ali et al.~\cite{Aden-AliFOCS23}, it was shown that if one combines the output of $\Omega(n)$ predictions made by one-inclusion algorithms on prefixes of the training data $((X_1,f^\star(X_1)),\ldots,(X_m,f^\star(X_m)))$ for $m=n/2,\dots,n$ via a majority vote, then the resulting function is optimal in the high probability regime and, therefore, matches the error bound \cref{eq:optimal_bound}.
Unfortunately, the one-inclusion graph algorithm (and this extension) is much less intuitive than the aforementioned algorithms based on taking majority votes of ERMs.

\subsection{The simplest possible optimal algorithm?}
In light of prior work, we have several provably optimal algorithms for PAC learning in the realizable setting. The algorithms and their analyses vary in complexity and a natural question remains: What is the simplest possible optimal algorithm? We know from lower bounds that the algorithm has to be improper and as such must be more complicated than ERM. Bagging is arguably the simplest algorithm among previous proposed algorithms, but has the most difficult analysis. The voting among one-inclusion algorithms has a somewhat simple proof, but the algorithm is not the simplest. 
In this work, we consider what is perhaps the simplest imaginable improper algorithm, \emph{Majority-of-Three} (\emph{ERMs}): Partition $S$ into three equal-sized disjoint pieces $S_1,S_2,S_3$, run the same ERM algorithm on each $(S_i,f^\star(S_i))$ to obtain $\widehat{f}_{S_1}, \widehat{f}_{S_2}, \widehat{f}_{S_3}$, and combine them via a majority vote to produce the function $\maj{(\widehat{f}_{S_1},\widehat{f}_{S_2},\widehat{f}_{S_3})}$. 
Since a majority vote of two functions is undefined when the functions disagree, this is arguably the simplest possible improper algorithm. 
Our first main result shows that this concrete majority vote of three ERMs, which we will refer to as Majority-of-Three throughout, is optimal in expectation.
\begin{restatable}{thm2}{optimalexpectation}\label{expextationboundsection:theorem}
Fix a function class $\mc{F} \subseteq \{0,1\}^{\mc{X}}$ with VC dimension $d$.
Fix a distribution $P$ over $\mc{X}$ and target function $f^\star \in \mc{F}$. 
For any ERM algorithm $\widehat{f} : \mc{X} \times \mc{Z}^* \to \{0,1\}$ it follows that
\[
\E_{S_1, S_2, S_3 \sim P^{n} }\left[ \err{\maj(\widehat{f}_{S_1},\widehat{f}_{S_2},\widehat{f}_{S_3})}{P} \right] = O\left(\frac{d}{n}\right).
\]
\end{restatable}
This result shows that Majority-of-Three matches the optimal expectation bound~\cref{eq:OIG_risk} achieved by the one-inclusion graph algorithm, up to a universal constant. Furthermore, our proof of Theorem~\ref{expextationboundsection:theorem} is in fact quite simple, especially compared to the previous proof that Bagging is optimal.

We note here that a single ERM alone is sub-optimal by a multiplicative $\ln(n/d)$ factor in-expectation (see the well-known lower bound in \cite[Theorem 4.2]{haussler1994predicting}).
We emphasize that in \cref{expextationboundsection:theorem}, the ERMs corresponding to $S_1$, $S_2$ and $S_3$ can be chosen by \emph{any} algorithm $\widehat{f}$ that outputs functions consistent with the sample. The only restriction is that it is the same algorithm $\widehat{f}$ that is run on each $S_i$ (and that the subsets $S_i$ are disjoint and thus i.i.d.).
 
We now turn our attention to the high-probability regime, where we prove the following result.
\begin{restatable}{thm2}{highprobboundsectiontheorem}\label{highprobboundsection:theorem}
Fix a function class $\mc{F} \subseteq \{0,1\}^{\mc{X}}$ with VC dimension $d$. 
Fix a distribution $P$ over $\mc{X}$ and target function $f^\star \in \mc{F}$. 
Fix any ERM algorithm $\widehat{f} : \mc{X} \times \mc{Z}^* \to \{0,1\}$.
For any parameter $\delta\in(0,1/2]$ it holds with probability at least $1-\delta$ over the randomness of $S_1,S_2,S_3\sim P^{n}$ that
\[
 \err{\maj(\widehat{f}_{S_1},\widehat{f}_{S_2},\widehat{f}_{S_3})}{P} = O \left(\frac{d}{n}\log\left(\log\left(\min\left\{\frac{n}{d},\frac{1}{\delta}\right\}\right)\right)+\frac{1}{n}\log\left(\frac{1}{\delta}\right)\right).
\]
\end{restatable}
This bound is sub-optimal due to the $\log (\log(\min\{n/d,1/\delta\}))$ term, 
however the additive $\log(1/\delta)$ term dominates for $\delta \le d^{-d}$.
Thus, Majority-of-Three is optimal both in the constant (\cref{expextationboundsection:theorem}) and high-probability regimes (\cref{highprobboundsection:theorem}). 
Because of this, we conjecture that Majority-of-Three is in fact optimal for all $\delta$ and leave this as an open question for future research.
\begin{conjecture}
Fix a function class $\mc{F} \subseteq \{0,1\}^{\mc{X}}$ with VC dimension $d$. 
Fix a distribution $P$ over $\mc{X}$ and target function $f^\star \in \mc{F}$. 
Fix any ERM algorithm $\widehat{f} : \mc{X} \times \mc{Z}^* \to \{0,1\}$.
For any parameter $\delta\in(0,1)$ it holds with probability at least $1-\delta$ over the randomness of $S_1,S_2,S_3\sim P^{n}$ that
\[
\err{\maj(\widehat{f}_{S_1},\widehat{f}_{S_2},\widehat{f}_{S_3})}{P} = O \left(\frac{d}{n}+\frac{1}{n}\log\left(\frac{1}{\delta}\right)\right).
\]    
\end{conjecture}
\subsection{An alternative by Simon}
In his breakthrough work, Simon~\cite{simon2015almost} proposed taking majority votes of three ERMs trained on certain sub-samples of the training sample.\footnote{Simon studied majority votes over any odd number $L$ of ERMs trained on specific sub-samples of the data.
He also proved bounds on the error of these majority votes that shrunk as $L$ increased.}
However, his algorithm is slightly different than ours. 
Concretely, he proposed the following algorithm: given an ERM algorithm and labelled training sample, partition $S$ into three equal-sized disjoint pieces $S_1,S_2,S_3$ and for $i=1,2,3$, run any ERM algorithm on $((S_1,\ldots,S_i),f^\star((S_1,\ldots,S_i)))$ to obtain $\widehat{f}_{S_1}, \widehat{f}_{(S_1,S_2)}, \widehat{f}_{(S_1, S_2, S_3)}$, and combine them via a majority vote to produce the function $\maj{(\widehat{f}_{S_1}, \widehat{f}_{(S_1,S_2)}, \widehat{f}_{(S_1, S_2, S_3)})}$. 
Intuitively, more training data for the ERM should be better and Simon also proved the following high-probability upper bound on his algorithm's error:
\begin{equation}\label{eq:Simon_majority}
\err{\maj\left(\bad_{S_1}, \bad_{(S_1,S_2)}, \bad_{(S_1, S_2, S_3)}\right)}{P} = O\left(\frac{d}{n}\log\left(\log\left(\frac{n}{d}\right)\right) + \frac{1}{n}\log\left(\frac{1}{\delta}\right)\right).
\end{equation}
This bound is asymptotically smaller than the tight bound~\cref{eq:erm_error} that holds for a single ERM.

We note that Simon also discusses the applicability of his analysis to more general majorities of ERMs including the Majority-of-Three function $\maj(\widehat{f}_{S_1}, \widehat{f}_{S_2}, \widehat{f}_{S_3})$ analyzed in this work.\footnote{Simon's analysis applies to any majority where each of the participating ERMs is trained on an independent constant fraction of the training sample.}
However, adopting the approach in \cite{simon2015almost}, the error of Majority-of-Three is controlled by the same upper bound as expressed in \eqref{eq:Simon_majority}, which is suboptimal as demonstrated by \cref{expextationboundsection:theorem}. 
Furthermore, we additionally remark that a similar in spirit construction based on the majority of three functions has been extensively studied in Schapire's PhD thesis \cite{schapire1992design}. 
However, his approach (inspired by what we now know as boosting) works with essentially any learning algorithm and is not necessarily limited to ERM. 

In the same work \cite{simon2015almost}, Simon further showed that for a specific function class $\mathcal{F}$ for which there is a choice of target function $f^\star \in \mc{F}$ and hard distribution $P$ that certify the tightness of \cref{eq:erm_error} for a certain choice ERM, his algorithm can actually achieve an optimal upper bound matching \cref{eq:improper_lowerbound} for $\mc{F}$ regardless of the choice of $f^\star \in \mc{F}$ and $P$.
Unfortunately, we prove the following lower bound that shows that the upper bound \cref{eq:Simon_majority} cannot be improved in general, answering a question posed by Simon.

\begin{restatable}{thm2}{lowerbound}\label{thm:lowerbound}
For any sample size $n$ that is divisible by $6$ and positive integer $d \le n$, there is a function class $\mc{F} \subseteq \{0,1\}^{[0,1]}$ with VC dimension $4d$, distribution $P$ over $[0,1]$, target function $f^\star \in \mc{F}$, and an ERM algorithm $\bad : \mc{X} \times \mc{Z}^* \to \{0,1\}$ such that the following holds: given i.i.d.\ training samples $S_1,S_2,S_3 \sim P^{n}$,
\[
\err{\maj\left(\bad_{S_1}, \bad_{(S_1,S_2)}, \bad_{(S_1, S_2, S_3)}\right)}{P} = \Omega\left(\frac{d}{n}\log\left(\log\left(\frac{n}{d}\right)\right)\right),
\]
with probability at least $2/3$ over the randomness of $S=(S_1,S_2,S_3)$.
\end{restatable}
This result shows that Simon's algorithm unfortunately cannot achieve the optimal bound \cref{eq:optimal_bound} in general.
This indicates that it is important that the ERM algorithm used in Majority-of-Three is instantiated on disjoint subsets of the training sample.

\subsection{Notation}\label{notationsection}
We use $\mathcal{X}$ to denote the \emph{instance space}, $\fc \subseteq \{0,1\}^{\mc{X}}$ to denote a function class,
and let $\mc{Z} = \mc{X} \times \{0,1\}$. 
Throughout, $P$ is a distribution over $\mc{X}$ and $f^\star \in \mc{F}$ is the unknown \emph{target function} in the class.
For $n \in \mathbb{N}$ and a distribution $P$, we denote by $P^n$ the product distribution of $P$.
We say that a sequence $S=(X_1\ldots,X_n)$ is a \emph{training sample} of size $n$ where  $X_i$ are i.i.d.\ samples from a distribution $P$.
For a training sample $S = (X_1, \dots, X_n)$, we find it useful to write $(S, f^\star(S)) = ((X_1,f^\star(X_1)),\ldots,(X_n,f^\star(X_n)))$, and we call this the \emph{labelled training sample}. 
For training samples $S_1=(X_1,\ldots,X_n)$ and $S_2=(X_{n+1},\ldots,X_{n+m})$ we let $(S_1,S_2)=(X_1,\ldots,X_n,X_{n+1},\ldots,X_{n+m})$, and for $S_1,$ $S_2$ and $S_3$ we take $(S_1,S_2,S_3)=((S_1,S_2),S_3)$.
We define the error of a binary function $f$ under distribution $P$ and target function $f^\star$ to be $\err{f}{P} = \Pr_{X \sim P}\left[ f(X) \not = f^\star(X) \right]$.
For any measurable set $R \subseteq \mathcal{X}$, we define $P_R$ to be the conditional distribution of $P$ restricted to $R$, i.e.\ for $X\sim P_R$ we have that for any measurable function $g$ that
$\E_{X\sim P_R}\left[g(X)\right]=\E_{X\sim P}\left[g(X)\mathbf{1}\{{X\in R}\}\right]/\Pr_{X\sim P}\left[X\in R\right].$

For a function class $\mathcal{F}$ and subset $U=\{x_1,\ldots,x_d\}\subseteq \mathcal{X}$ of $d$ points we let $\mathcal{F} \mid_{U}$ denote the set $\{y\in \{0,1\}^d\mid \exists f\in \mathcal{F}: \forall i\in[d], \; f(x_i)=y_i \}$.
The \emph{Vapnik-Chervonenkis (VC) dimension} is then defined as the largest number $d$ such that there exists a point set $U\subseteq \mathcal{X}$ of size $d$ such that the cardinality of $\mathcal{F}\mid_U$ is $2^d$. We use $\log(x)$ and $\ln(x)$ to denote $\log_2(x)$ and $\log_e(x)$ respectively and we also use $\Log(x):=\max\{2,\log_2(x)\}$ to denote a truncated logarithm.

Let $\mc{Z}^* = \cup_{i=1}^{\infty} \mc{Z}^i$ be the set of all possible labelled training samples.
We define a \emph{learning algorithm} $\widehat{f}$ to be a mapping $\widehat{f} : \mc{X} \times \mc{Z}^* \to \{0,1\}$.
That is, given a labelled training sample $(S,f^\star(S))$ as input, $\widehat{f}(\cdot ; (S,f^\star(S))) : \mc{X} \to \{0,1\}$ is the function that is learned from $(S,f^\star(S))$.
For ease of reading, we often denote the learned function by $\widehat{f}_S \coloneqq \widehat{f}(\cdot ; (S,f^\star(S)))$.
A learning algorithm $\widehat{f}$ is an \emph{Empirical Risk Minimizer (ERM)} for the class $\mc{F}$ if, given a labelled training sample $(S,f^\star(S))$ as input, it output a function $\widehat{f}_S$ in $
\mc{F}$ that satisfies $\widehat{f}_S(X_i) = f^\star(X_i)$ for every $X_i$ that appears in $S$.
We define the majority vote of $k$ binary functions $f_1, \dots, f_k : \mc{X} \to \{0,1\}$ to be the function 
\[
\maj(f_1, \dots, f_k)(x) \coloneqq \mathbf{1}\{ f_1(x) + \dots + f_k(x) > k/2\}.
\]
\section{Majority-of-Three is optimal in-expectation}\label{expectationupperboundsection}
In this section, we prove the main in-expectation result for the Majority-of-Three algorithm.
Before we prove our result, we will find it helpful to introduce some auxiliary notation. 
Throughout this section, we set $\widehat{f} : \mc{X} \times \mc{Z}^* \to \{0,1\}$ to be a fixed (but arbitrary) ERM algorithm.
Fix a distribution $P$ over $\mathcal{X}$ and let $f^\star \in \mc{F}$ be the target function.
For any $x \in \mathcal{X}$ we let 
\[
p_x(n,f^\star,P) = \Pr_{S \sim P^{n}}[\widehat{f}_S(x)\not=f^\star(x)].
\]
In words, $p_x(n,f^\star,P)$ is the chance that $\widehat{f}_S$ errs on the point $x$ for an \emph{average} sample $S \sim P^n$.
We now define a partition of $\mathcal{X}$ based on $p_x(n,f^\star,P)$. 
Consider the following sets for any $i\in\mathbb{N}$:
\[
R_i(n,f^\star,P)=\{x\in\mathcal{X}: p_x(n,f^\star,P) \in (2^{-i},2^{-i+1}]\}.
\]
We often write $R_i = R_i(n,f^\star,P)$ and $p_x = p_x(n,f^\star,P)$ since $n$, $P$, and $f^\star$ will always be clear from the context.
With this notation in place, we are now ready to prove that Majority-of-Three has an optimal in-expectation  
upper bound on its error.
\optimalexpectation*

To prove~\cref{expextationboundsection:theorem}, we require the following lemma which says that two ERMs trained on 2 i.i.d.\ training samples of the same size rarely makes a mistake on the same point.
\begin{lemma}\label{lem:joint_mistake}
Fix a function class $\mc{F} \subseteq \{0,1\}^{\mc{X}}$ with VC dimension $d$.
Fix a distribution $P$ over $\mc{X}$ and target function $f^\star \in \mc{F}$. 
For any ERM algorithm $\widehat{f} : \mc{X} \times \mc{Z}^* \to \{0,1\}$ it follows that
\[
\E_{S_1, S_2 \sim P^{n} }\left[\Pr_{X \sim P}\left[\widehat{f}_{S_1}(X) \not = f^\star(X) \wedge \widehat{f}_{S_2}(X) \not= f^\star(X) \right]\right] \le c\frac{d}{n},
\]
where $c$ is a universal constant.
\end{lemma}
We postpone the proof of \Cref{lem:joint_mistake} for now and show how it implies \Cref{expextationboundsection:theorem}.
\begin{proof}[Proof of \Cref{expextationboundsection:theorem}] 
For any fixed $x \in \mathcal{X}$ and fixed samples $S_1, S_2, S_3$, if $\maj(\widehat{f}_{S_1},\widehat{f}_{S_2},\widehat{f}_{S_3})(x) \not = f^\star(x)$, then there must be at least two distinct indices $i, j \in [3]$ such that $\widehat{f}_{S_i}(x) \not = f^\star(x)$ \emph{and} $\widehat{f}_{S_j}(x) \not = f^\star(x)$. So,
\begin{align*}
     \err{\maj(\widehat{f}_{S_1},\widehat{f}_{S_2},\widehat{f}_{S_3})}{P}  &=    \Pr_{X \sim P}\left[ \maj(\widehat{f}_{S_1},\widehat{f}_{S_2},\widehat{f}_{S_3})(X) \not = f^\star(X) \right]  \\
    &\le \sum_{\substack{i,j \in [3]\\ i<j}} \Pr_{X \sim P}\left[\widehat{f}_{S_i}(X) \not = f^\star(X) \wedge \widehat{f}_{S_j}(X) \not= f^\star(X) \right].
    \end{align*}
Combining the above and \cref{lem:joint_mistake} gives us 
\[
\E_{S_1, S_2, S_3 \sim P^{n} }\left[\err{\maj(\widehat{f}_{S_1},\widehat{f}_{S_2},\widehat{f}_{S_3})}{P}\right] \le 3c\frac{d}{n}.
\]
This concludes the proof.
\end{proof}

We now move on to proving \cref{lem:joint_mistake}, where we will use the following lemma.
\begin{restatable}{lem2}{conditional}\label{lem:expected_uniform_conditional}
Fix a function class $\mc{F} \subseteq \{0,1\}^{\mc{X}}$ with VC dimension $d$.
Fix a distribution $P$ over $\mathcal{X}$, target function $f^\star \in \mc{F}$, and $R \subseteq \mc{X}$ such that $\Pr_{X \sim P}\left[X \in R\right] \not= 0$.
For any ERM algorithm $\widehat{f} : \mc{X} \times \mc{Z}^* \to \{0,1\}$ it follows that
\[
\E_{S \sim P^n}\left[\err{\widehat{f}_S}{P_R}\right]\leq 20\frac{d\Log(e\Pr_{X \sim P}[X \in R]n/d)}{\Pr_{X \sim P}[X \in R]n}.
\]
\end{restatable}

\cref{lem:expected_uniform_conditional} is an immediate consequence of the celebrated uniform convergence principle and a simple proof can be found in~\cref{lem:expected_uniform_conditionalsubsection}. 
We now prove \cref{lem:joint_mistake}.
\begin{proof}[Proof \cref{lem:joint_mistake}]
Let $S_1$ and $S_2$ be independent samples from $P^n$.
By the independence of $S_1$ and $S_2$ and the definition of $p_x$ we have, for any $x \in \mc{X}$, that
\begin{align*}
&\Pr_{S_1, S_2 \sim P^{n} }\left[\widehat{f}_{S_1}(x) \not = f^\star(x) \wedge \widehat{f}_{S_2}(x) \not= f^\star(x) \right]=\prod_{i=1}^2\Pr_{S_i \sim P^{n} }\left[\widehat{f}_{S_i}(x) \not = f^\star(x) \right]\\&=\Pr_{S_1\sim P^{n} }\left[\widehat{f}_{S_1}(x) \not = f^\star(x) \right]^2=p_x^2.    
\end{align*}

Using the above, the law of total expectation with partitioning $(R_i)_{i\in\mathbb{N}}$, and swapping the order of expectations ($X$ and $(S_1,S_2)$ being independent), we get that 
\begin{align*}
    &\E_{S_1, S_2 \sim P^{n} }\left[\Pr_{X \sim P}\left[\widehat{f}_{S_1}(X) \not = f^\star(X) \wedge \widehat{f}_{S_2}(X) \not= f^\star(X) \right]\right]
\\&=\sum_{i=1}^{\infty} \Pr_{X \sim P}\left[X\in R_i\right]\E_{X \sim P}\left[p_X^2|X\in R_i\right]\leq\sum_{i=1}^{\infty} \Pr_{X \sim P}\left[X\in R_i\right]2^{-2i+2},
\end{align*}
where the inequality follows from the fact that $p_x \le 2^{-i+1}$ for every $x\in R_{i}$.
We will now show that $\Pr_{X \sim P}\left[X\in R_i\right]\leq c di2^{i}/n $ for every $i\in\mathbb{N}$ (for a universal constant $c\geq 1$ chosen below), which combined with the above gives us
\begin{align*}
\E_{ S_1, S_2 \sim P^{n} }\left[\Pr_{X \sim P}\left[\widehat{f}_{S_1}(X) \not = f^\star(X) \wedge \widehat{f}_{S_2}(X) \not= f^\star(X) \right]\right]\leq \frac{4cd}{n}\sum_{i=1}^{\infty} i2^{-i}\leq 8c\frac{d}{n}.
\end{align*}
This yields the claim with the constant $8c$.
Towards a contradiction, assume there is an $i \in \mathbb{N}$ such that $\Pr_{X \sim P}\left[X\in R_i\right]> cdi2^{i}/n $, which is equivalent to $\Pr_{X \sim P}\left[X\in R_i\right]n/d> ci2^{i}\geq 1 $.
Using this assumption, the fact that $x\rightarrow \Log(ex)/x$ is decreasing for $x > 0$, and \cref{lem:expected_uniform_conditional},
we have
\begin{equation}\label{eq:not_simple}
\E_{S_1 \sim P^{n}}\left[\err{\erm{1}}{P_{R_i}}\right] \le 20\frac{\Log(e\Pr_{X \sim P}\left[X\in R_i\right]n/d)}{(\Pr_{X \sim P}\left[X\in R_i\right]n/d)}\leq 20\frac{\Log\left(eci2^i\right)}{ci2^i}.
\end{equation}
By changing the order of expectations of the left hand side of the above and using the fact that $p_x > 2^{-i}$ for every $x \in R_i$, we also have
\begin{equation}
\E_{S_1 \sim P^n} \left[\err{\widehat{f}_{S_1}}{P_{R_i}}\right] = \E_{X \sim P_{R_i}} \left[ p_{X}\right]  > 2^{-i} . \label{eq:simple}
\end{equation}
Combining the upper bound \eqref{eq:not_simple}, the lower bound \eqref{eq:simple}, and the fact that the function $x\rightarrow \Log(ex)/x$ is decreasing for $x>0$, we get
\begin{align*}
1 < 20\frac{\Log\left(eci2^i\right)}{ci}\leq 20\left(\frac{\Log\left(eci\right)}{ci}+\frac{2}{c}\right)
\leq 20\frac{\Log(ec)+2}{c}.
\end{align*}
However, for $c$ large enough, the right hand side of the above is strictly less than $1$.
This gives us the desired contradiction and concludes the proof.
\end{proof}
\section{A lower bound for certain majorities}
In this section, we prove that not all majorities of 3 ERMs trained on subsets of the data are optimal. 
In particular, we show that Simon's~\cite{simon2015almost} original partitioning scheme of the training sample into 3 sub-samples can produce a majority of 3 ERMs with sub-optimal error. 
Recall Simon's algorithm: partition the training sample $S = (S_1, S_2, S_3)$ into 3 equal pieces $S_1, S_2, S_3$, train 3 ERMs $\widehat{f}_{S_1}, \widehat{f}_{(S_1,S_2)}, \widehat{f}_{(S_1,S_2,S_3)}$, and return the majority vote $\maj(\widehat{f}_{S_1}, \widehat{f}_{(S_1,S_2)}, \widehat{f}_{(S_1,S_2,S_3)})$.
Simon proved that this algorithm enjoys the PAC upper bound 
\[
\err{\maj\left(\bad_{S_1}, \bad_{(S_1,S_2)}, \bad_{(S_1, S_2, S_3)}\right)}{P} = O\left(\frac{d}{n}\log\left(\log\left(\frac{n}{d}\right)\right) + \frac{1}{n}\log\left(\frac{1}{\delta}\right)\right).
\]
The next theorem shows that this algorithm unfortunately has a matching lower bound on its error.
\lowerbound*
Comparing the above bound with the upper bound in \cref{expextationboundsection:theorem}, we see that if the ERMs did not overlap in their sub-samples the log factor would not be present.
The construction we use in our lower bound is a modification of the usual construction used to prove a lower bound on the error of a single ERM (see~\cite{auer2007new,simon2015almost}).
In these constructions, one takes the domain $\mc{X}$ to be a finite set of size roughly $n/\log(n/d)$ where $n \ge d$ is the sample size\footnote{These results are often stated as lower bounds on the \emph{sample complexity} for some target error $\epsilon$.} and the function class $\mc{F}$ is taken to be all functions that assign the value $1$ to at most $d$ points on $\mc{X}$.
Furthermore, the target function is set to be the $0$ function, and the sampling distribution is the uniform distribution over $\mc{X}$.
Finally, the ``bad'' ERM algorithm returns any function that assigns as many 1's to the domain as possible, while  being consistent on the observed samples.
The error of this ERM is tightly connected to the number of unique elements we sample from the domain.
One can then use a coupon collector argument to show that the error is $\Omega(d\log(n/d)/n)$ with constant probability.

Simon noticed that we cannot directly use this ``hard instance'' to prove a lower bound on his algorithm due to the structure of the class $\mc{F}$~\cite[Theorem 7]{simon2015almost}.
We get around this by considering a version of this construction that uses a continuous domain (instead of finite) and a function class consisting of functions that are unions of intervals (instead of points).

Before we prove \cref{thm:lowerbound}, it will be convenient to introduce the following notation. 
For a positive integer $d$ and non-empty set $A$, we define $A_{\round}$ to be the set consisting of the smallest $d$ elements of $A$ with respect to an ordering of the elements of $A$. The ordering we use will be clarified when needed.
We now prove \cref{thm:lowerbound}.
\begin{figure}[ht]
    \centering
    \bigskip
    \input{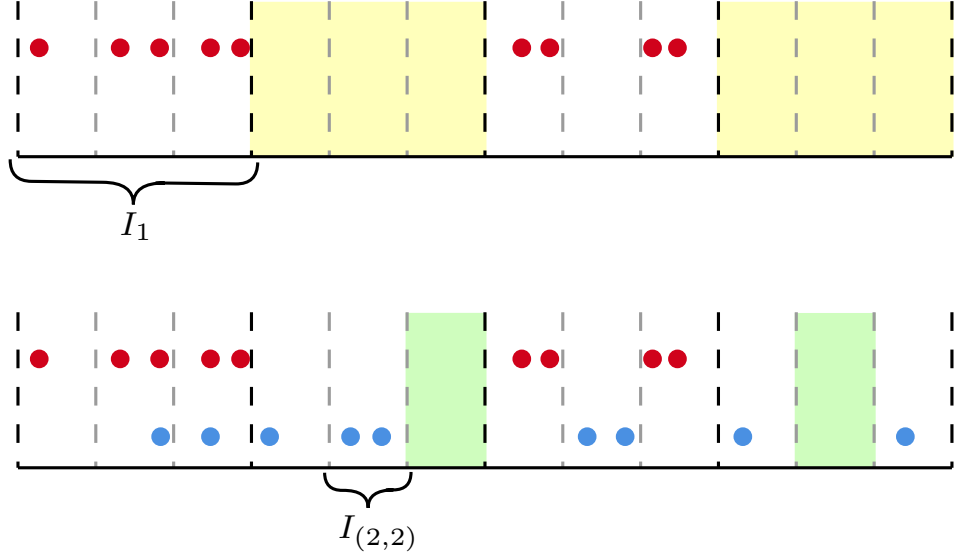}
    \caption{
    An illustration of the partitioning of the interval $(0,1]$ for a training sample consisting of $m = 18$ points with $d=2$.
    The interval $(0,1]$ is partitioned into $4$ intervals $I_{1}, \dots, I_{4}$.
    Each interval $I_{i}$ is further partitioned into the $4$ subintervals $I_{(i,1)}, \dots, I_{(i,4)}$.
    The red points correspond to the first half of the sample $(X_1, \dots, X_{9} )$ and the blue points correspond to the second half of the sample $(X_{10}, \dots, X_{18} )$.
    The yellow highlighted regions are the first $d$ intervals $I_2$ and $I_4$ that contain no points from $(X_1, \dots, X_{9})$.
    The green highlighted regions are the first $d$ subintervals of $I_2$ and $I_4$ that contain no points from $(X_{10}, \dots, X_{18})$.
    The green intervals are added to the union of intervals used by $\widehat{f}_S$ as their indices correspond to the set $L_1(S)$.
    }
    \label{fig:lowerbound}
    \label{fig:lowerboundv2}
\end{figure}
\begin{proof}[Proof of \cref{thm:lowerbound}]
Fix a sample size $n$ divisible by $6$ and positive integer $d \le n$.
Throughout, we will assume any interval considered is left-open and right-closed.
A collection of intervals $I_1, \dots, I_t \subseteq [0,1]$ can be viewed as the binary function $f_{I_1 \cup \dots \cup I_t}$ that satisfies $f_{I_1 \cup \dots \cup I_t}(x) = 1$ if and only if there is an index $j$ such that $x \in I_j$.
We will consider the function class $\mc{F}$ that is the collection of all functions corresponding to the union of at most $2d$ interval.
It is not hard to show that this class has VC dimension $4d$.
We take $P$ to be the uniform distribution on the domain $[0,1]$ and choose the target function $f^\star$ to be the $0$ function on the domain $[0,1]$.

We now describe the ``bad'' ERM algorithm $\bad : \mc{X} \times \mc{Z}^* \to \{0,1\}$.
For the remainder of the proof, $\mega > 0$ is a large universal constant that we will determine below.
For a training sample size $m$, we define three collections of sets:
\begin{enumerate}
    \item $\{I_i(m) : i \in [m_1]\}$ is the unique partition of $(0,1]$ into $m_1 \coloneqq \lceil \mega m/\ln\left(\mega m/d\right) \rceil$ intervals of the same length. 
    \item $\{I_{i,j}(m) : i \in [m_1], j \in [m_2]\}$ where, for a fixed $i$, $\{I_{i,j}(m) : j \in [m_2]\}$ is the unique partition of $I_i(m)$ into $m_2 \coloneqq \lceil 4\mega m/\left(m_1\ln\left(\ln\left(\mega m/d\right) \right)\right) \rceil$ intervals of the same length.
    \item $\{J_i(m) : i \in m_{3}\}$ is the unique partition of $(0,1]$ into $m_3 \coloneqq \lceil 2\mega m/\ln\left(2\mega m/d\right) \rceil$ intervals of the same length. 
\end{enumerate}
Given a labelled training sample $(S,f^\star(S)) = ((X_1, 0), \dots, (X_m,0))$ as input, the ERM algorithm $\bad$ constructs the function $\bad_S = \bad(\cdot ; (S,f^\star(S)))$ in the following way:\footnote{This defines $\bad_S$ when $(S,f^\star(S))$ contains only $0$ labels. On any $(S,f^\star(S))$ that contains a $1$ label we return an arbitrary consistent function.}
\begin{enumerate}
    \item For $i \in [m_1]$, $j \in [m_2]$, and $k \in [m_3]$ define the sets
    \begin{gather*}  
    \tilde{I}_i(S) = \{x_1, \dots, x_{\lfloor m/2 \rfloor} \} \cap I_i(m), \\
    \tilde{I}_{(i,j)}(S) = \{x_{\lfloor m/2 \rfloor + 1} , \dots,  x_m \} \cap I_{(i,j)}(m), \\
    \tilde{J}_k(S) = \{x_1, \dots, x_m \} \cap J_{k}(m).
    \label{eq:vert-set}
\end{gather*}
    \item Define the index sets
    \begin{gather*}
        L_1(S) =\{(i,j): i \in \{i' : \tilde{I}_{i'}(S) =\emptyset \}_{\round}, \tilde{I}_{(i,j)}(S) = \emptyset \}_{\round},\footnotemark\\
        L_2(S)  = \{k: \tilde{J}_k(S) =\emptyset\}_{\round}.
    \end{gather*}
    \footnotetext{The ordering used for pairs $(i,j)$ and $(i',j')$ is the natural one: $(i,j) \le (i',j')$ if $i < i'$ or $i=i'$ and $j \le j'$.}
    \item Define the union of intervals \[
    I_S = \left( \bigcup_{(i,j) \in L_1(S)} I_{i,j}(m)\right) \bigcup  \left(\bigcup_{i \in L_2(S)} J_i(m)\right).
    \]
    \item Finally, define the function $\bad_{S} = f_{I_S}$.
\end{enumerate}
Observe that $I_S$ is the union of at most $2d$ disjoint intervals, so $\bad_S$ will always be in the class $\mc{F}$.
Furthermore, $\bad_S$ is always consistent with the sample $S$ by construction.
See~\cref{fig:lowerboundv2} for an example of the resulting intervals considered by the set $L_1(S)$.
Let $m = n/3$.
From now on we use $m_1$ and $m_2$ to denote the number of intervals of the form $I_{i}(2m)$ and $I_{(i,j)}(2m)$ considered by $\bad_{(S_1, S_2)}$ respectively.
Consider the unions of intervals $I_{S_1}$ and $I_{(S_1, S_2)}$ corresponding to the ERM functions $\bad_{S_1}$ and $\bad_{(S_1, S_2)}$. The number $m$ is divisible by $2$ from our choice of $n$, so it follows that $J_i(m) = I_i(2m)$ and $\tilde{J}_i(S_1) = \tilde{I}_i(S_1, S_2)$, which implies $L_2(S_1)=\{k: \tilde{I}_k((S_1,S_2)) =\emptyset\}_{\round}$.
Thus, $\bad_{S_1}$ and $\bad_{(S_1,S_2)}$ agree, and simultaneously err, on every subinterval $I_{(i,j)}(2m)$ with $(i,j) \in L_1(S_1, S_2)$.
Because $P$ is the uniform distribution and every interval $I_{(i,j)}(2m)$ has length $1/(m_1 m_2)  = \Theta( \ln \left(\ln \left( n/d\right)\right)/ n )$,
it follows that the error of the majority vote satisfies
\[
\err{\maj\left(\bad_{S_1}, \bad_{(S_1, S_2)}, \bad_{(S_1, S_2, S_3)}\right)}{P} \ge \frac{|L_1(S_1, S_2)|}{m_1 m_2} = \Omega\left( \frac{ |L_1(S_1, S_2)|}{n} \ln \left(\ln \left(\frac{n}{d}\right)\right) \right).
\]
Thus, if $|L_1(S_1, S_2)| = d$, we have 
\[
\err{\maj\left(\bad_{S_1}, \bad_{(S_1, S_2)}, \bad_{(S_1, S_2, S_3)}\right)}{P} = \Omega\left(\frac{d}{n}\log\left(\log\left(\frac{n}{d}\right)\right)\right),
\]
so the claim of the theorem follows once we prove that 
\[
\Pr_{(S_1, S_2) \sim P^{2m}}\left[|L_1(S_1, S_2)| = d \right] \ge 2/3.
\]
To this end, let $E_1 = E_1(S_1)$ be the event that $S_1$ satisfies $|L_2(S_1)| = d$ and let $E_2 = E_2((S_1,S_2))$ be the event that $(S_1, S_2)$ satisfies $|L_1(S_1, S_2)| = d$. 
Using the law of total probability we get,
\[
\Pr_{(S_1, S_2) \sim P^{2n/3}}\left[|L_1(S_1, S_2)| = d \right] =  \Pr_{(S_1, S_2) \sim P^{2m}}\left[E_2 \right] \ge \Pr_{(S_1, S_2) \sim P^{2m}}\left[E_2  \mid E_1\right] \Pr_{S_1 \sim P^{m}}\left[ E_1\right],
\]
so it suffices to prove that $\Pr_{(S_1, S_2) \sim P^{2m}}\left[E_2  \mid E_1\right] \ge \sqrt{2/3}$ and $\Pr_{S_1 \sim P^{m}}\left[ E_1\right] \ge \sqrt{2/3}$.
We omit the proof of the later inequality since it is very similar to the proof of the former inequality.

When $E_1$ occurs, we have $|L_2(S_1)| =|\{k: \tilde{J}_k(S_1) =\emptyset\}_{\round}|=|\{k: \tilde{I}_k((S_1,S_2)) =\emptyset\}_{\round}|= d$.
So, showing that the event $E_2$ occurs conditioned on $E_1$ is equivalent to showing that \emph{at least} $d$ subintervals in the collection $\{I_{(i,j)} : i \in L_2(S_1), j \in [m_2] \}$ do not contain any points from the sample $S_2$.
Let $Y \sim Q$ be the random variable that counts the number of points required to sample from $P$ until $m_2d -d$ subintervals in $\{ I_{(i,j)} : i \in L_2(S_1), j \in [m_2] \}$ contain one of the sampled points.
Furthermore, let $Y_t \sim Q_t$ denote the random variable that counts the number of trials required to cover $(t+1)$ subintervals given that we have covered $t$. 
Notice that $Y_t$ is a geometric random variable with parameter $p_t = \frac{m_2 d - t}{m_1 m_2} = \frac{d}{m_1} - \frac{t}{m_1 m_2}$ and $Y = \sum_{t=0}^{m_2 d - d - 1} Y_t$.
It follows that
\[
\Pr_{(S_1, S_2) \sim P^{2m}}\left[E_2  \mid E_1\right]  \ge \Pr_{Y \sim Q}\left[Y \ge m\right] = \Pr_{Y_t \sim Q_t}\left[\sum_{t=0}^{m_2d -d -1} Y_t  \ge m\right].
\]
We can use a concentration inequality for sums of geometric random variables together with some simple calculations to show that 
\begin{equation}
\Pr_{Y_t \sim Q_t}\left[\sum_{t=0}^{m_2d -d -1} Y_t  \ge m\right] \ge \sqrt{2/3} ,\label{eq:sum_geometric}
\end{equation}
when $\mega$ is large enough.
We defer these calculations to \cref{app:proof_sum_geometric}.
This concludes the proof.
\end{proof}\label{lowerboundv3}
\section{High probability upper bound}\label{sec:highprobability}
In this section we prove our high-probability upper bound for the Majority-of-Three algorithm which we now restate for convenience.
\highprobboundsectiontheorem*
In this section it will be convenient to use the following notation: for a probability distribution $P$ over $\mc{X}$ and set $R \subseteq \mc{X}$, we define $P(R) = \Pr_{X\sim P}\left[X \in R\right]$.
\cref{highprobboundsection:theorem} is a consequence of the following technical lemma.
\begin{lemma}\label{highprobboundsection:lemma1}
 Fix a function class $\mc{F} \subseteq \{0,1\}^{\mc{X}}$ with VC dimension $d$. 
Fix a distribution $P$ over $\mc{X}$ and target function $f^\star \in \mc{F}$. 
Fix an ERM algorithm $\widehat{f} : \mc{X} \times \mc{Z}^* \to \{0,1\}$.
For any parameter $\delta\in(0,1/2]$ it holds with probability at least $1-\delta$ over the randomness of $S_1,S_2\sim P^{n}$ that
\[
\underset{X \sim P}{\operatorname{Pr}}\left[\widehat{f}_{S_1}(X) \neq f^{\star}(X) \wedge \widehat{f}_{S_2}(X) \neq f^{\star}(X)\right] \leq c \left(\frac{d}{n} \Log\left(\Log \left(\min\left\{\frac{n}{d},\frac{1}{\delta}\right\}\right)\right)+\frac{1}{n}\Log\left(\frac{1}{\delta}\right)\right),
\]
where $c$ is a universal constant.
\end{lemma}
 We now prove \cref{highprobboundsection:theorem} using \cref{highprobboundsection:lemma1} and postpone the proof of \cref{highprobboundsection:lemma1}.
\begin{proof}[Proof of \cref{highprobboundsection:theorem}]
Since $\maj(\widehat{f}_{S_1},\widehat{f}_{S_2},\widehat{f}_{S_3})(x)\not=f^{\star}(x)$ happens if and only if there exists two distinct indices $i,j\in [3]$ such that $\widehat{f}_{S_i}(X) \neq f^{\star}(X)$ \emph{and} $\widehat{f}_{S_j}(X) \neq f^{\star}(X)$, we get that
\begin{align*}
 \err{\maj(\widehat{f}_{S_1},\widehat{f}_{S_2},\widehat{f}_{S_3})}{P}\le  \sum_{\substack{i,j\in[3]\\ i<j}}\underset{X \sim P}{\operatorname{Pr}}\left[\widehat{f}_{S_i}(X) \neq f^{\star}(X) \wedge \widehat{f}_{S_j}(X) \neq f^{\star}(X)\right].
\end{align*}
Using~\cref{highprobboundsection:lemma1} with confidence parameter $\delta/3$ for every distinct pair $i,j\in [3]$ together with a union bound gives us, with probability at least $1-\delta$ over the randomness of $(S_1,S_2,S_3)$, that
\begin{align*}
\err{\maj(\widehat{f}_{S_1},\widehat{f}_{S_2},\widehat{f}_{S_3})}{P} = O \left(\frac{d}{n} \Log\left(\Log \left(\min\left\{\frac{n}{d},\frac{1}{\delta}\right\}\right)\right)+\frac{1}{n}\Log\left(\frac{1}{\delta}\right)\right).
\end{align*}
This concludes the proof.
\end{proof}
Before we prove \cref{highprobboundsection:lemma1}, we provide a short overview of the proof. 
Our first step is to reuse the idea from \cref{expectationupperboundsection} to partition the instance space $\mc{X}$ into sets $\{R_i\}_{i \in \mathbb{N}}$ based on the chance that an average ERM errs on a point in $x \in \mc{X}$.
However, we use a different way to quantify the errors defining $R_i$ by incorperating the failure parameter $\delta$.
For $i \ge 2$, we can actually reuse our in-expectation analysis from \cref{expectationupperboundsection} together with a simple application of Markov's inequality and a sequence of union bounds.
This gives us an upper bound on the joint error of two ERMs on the conditional distributions for all $\{R_i\}_{i \ge 2}$, with high probability.
The major technical work of the proof lies in controlling the joint error of two ERMs on the conditional distribution of $R_1$.
To do this, we borrow an idea from Simon~\cite{simon2015almost} that views the probability of $\widehat{f}_{S_1}$ and $\widehat{f}_{S_2}$ jointly erring as the probability that $\widehat{f}_{S_1}$ errs times the probability that $\widehat{f}_{S_2}$ errs conditioned on $\widehat{f}_{S_1}$ erring.\footnote{This idea used by Simon in fact builds upon even earlier work of Hanneke~\cite{Hanneke2009} in the context of active learning. It was also applied in the context of PAC learning by Darnst\"{a}dt \cite{darnstadt2015optimal}.}
A crucial technicality that differentiates our setting from Simon's is that the probability that $\widehat{f}_{S_1}$ and $\widehat{f}_{S_2}$ jointly err is taken over a \emph{conditional distribution} $P_R$ rather than the distribution $P$ from which the samples $S_1$ and $S_2$ are drawn.

The following lemma formalizes how we can control the joint error of two ERMs under $P_R$.
\begin{lemma}\label{condsimon}
Fix a function class $\mc{F} \subseteq \{0,1\}^{\mc{X}}$ with VC dimension $d$. 
Fix a distribution $P$ over $\mc{X}$, target function $f^\star \in \mc{F}$ and $R\subseteq \mathcal{X}$ such that $P(R) \not= 0$.
Fix an ERM algorithm $\widehat{f} : \mc{X} \times \mc{Z}^* \to \{0,1\}$.
For any parameter $\delta\in(0,1/2]$ it holds with probability at least $1-\delta$ over the randomness of $S_1,S_2\sim P^{n}$ that
\begin{align*}
\underset{X \sim P_R}{\operatorname{Pr}}\left[\widehat{f}_{S_1}(X) \neq f^{\star}(X) \wedge \widehat{f}_{S_2}(X) \neq f^{\star}(X)\right] \leq  8\max\left\{\frac{d\Log(8e\Log(eP\left(R\right)n/d))}{P\left(R\right)n},\frac{\Log(8/\delta)}{P\left(R\right)n}\right\}.
 \end{align*}
\end{lemma}
We now prove \cref{highprobboundsection:lemma1} using \cref{condsimon} and postpone the proof of \cref{condsimon}.
\begin{proof}[Proof of \cref{highprobboundsection:lemma1}]
We use the same definition for $p_x$ (see \cref{expectationupperboundsection}) but redefine the sets $\{R_i\}_{i \in \mathbb{N}}$ to be
\[
R_1=\left\{x \in \mathcal{X}: p_x \in(2^{-1}\delta/\Log(1/\delta), 1]\right\},
\]
and for any integer $i \geq 2$,
\[
R_i=\left\{x \in \mathcal{X}: p_x \in\left(2^{-i} \delta/\Log(1/\delta), 2^{-i+1} \delta/\Log(1/\delta)\right]\right\}.
\]
Using the law of total probability we have
\begin{align*}
\Pr_{X\sim P} \left[\widehat{f}_{S_1}(X) \neq f^{\star}(X) \wedge \widehat{f}_{S_2}(X) \neq f^{\star}(X)\right] &= P\left(R_1\right) \underset{X \sim P_{R_1}}{\operatorname{Pr}}\left[\widehat{f}_{S_1}(X) \neq f^{\star}(X) \wedge \widehat{f}_{S_2}(X) \neq f^{\star}(X)\right]\\
&\quad+\sum_{i=2}^{\infty} P\left(R_i\right) \underset{X \sim P_{R_i}}{\operatorname{Pr}}\left[\widehat{f}_{S_1}(X) \neq f^{\star}(X) \wedge \widehat{f}_{S_2}(X) \neq f^{\star}(X)\right].
\end{align*}
We will prove that there is a universal constant $c > 0$ such that the events
\begin{align*}
E_1 = E_1((S_1, S_2))\coloneqq\Bigg \{P\left(R_1\right) \underset{X \sim P_{R_1}}{\operatorname{Pr}}\left[\widehat{f}_{S_1}(X) \neq f^{\star}(X) \wedge \widehat{f}_{S_2}(X) \neq f^{\star}(X)\right] 
\leq\\ c \max\left\{\frac{d\Log(\Log(\min\{n/d,1/\delta\}))}{n},\frac{\Log(1/\delta)}{n}\right\}\Bigg \},
\end{align*}
and
\begin{align*}
E_2 = E_2((S_1, S_2))\coloneqq\left\{\sum_{i = 2}^{\infty} P\left(R_i\right) \underset{X \sim P_{R_i}}{\operatorname{Pr}}\left[\widehat{f}_{S_1}(X) \neq f^{\star}(X) \wedge \widehat{f}_{S_2}(X) \neq f^{\star}(X)\right] \le c\frac{d}{n}
\right \}
\end{align*}
each happen with probability at least $1-\delta/2$ over the randomness of $(S_1,S_2)$.
The claim of \cref{highprobboundsection:lemma1} then follows from a union bound. 
Define the set $I = \{ i \ge 2 : P(R_i) \not = 0\}$.
To prove that $E_1$ and $E_2$ each occur with high probability, we will use the following proposition.
\begin{restatable}{prop2}{helpermarkov}\label{helpermarkov}
In the setting of \cref{highprobboundsection:lemma1} we have the following: 
\begin{enumerate}
\item There is a universal constant $c'$ such that for any $i\in\mathbb{N}$ \label{item:1}
\begin{align*}P\left(R_i\right) \leq \frac{c' 2^id \Log (2^i\Log(1/\delta) / \delta) \Log(1/\delta)  }{\delta n}.
\end{align*}
\item With probability at least $1-\delta/2$ over the randomness of $(S_1,S_2)$  we have, simultaneously for all $i \in I = \{ i \ge 2 : P(R_i) \not = 0\}$, that \label{item:2}
\[
\underset{X \sim P_{R_i}}{\operatorname{Pr}}\left[\widehat{f}_{S_1}(X) \neq f^{\star}(X) \wedge \widehat{f}_{S_2}(X) \neq f^{\star}(X)\right] \leq \frac{5\cdot 2^{-1.1 i} \delta}{\Log^2(1/\delta)}.
\]
\end{enumerate}
\end{restatable}
\noindent We defer the proof of \cref{helpermarkov} to \cref{app:helpermarkov} as its proof is similar to that of \cref{lem:joint_mistake}.

We first prove that the event $E_1$ occurs with high probability. If $P(R_1) = 0$, then we immediately have that $\Pr_{(S_1, S_2)}[E_1] = 1$. We now consider the case that $P(R_1) \not= 0$.
From \cref{item:1} of \cref{helpermarkov} we can conclude there is a universal constant $\tilde{c}$ such that $P\left(R_1\right)\leq \min\{1,\tilde{c}\frac{d\Log^2(1/\delta)}{\delta n}\}$.
Using this combined with \cref{condsimon} we have, with probability at least $1-\delta/2$ over the randomness of $(S_1, S_2)$, that
 \begin{align*}
&P\left(R_1\right)\underset{X \sim P_{R_1}}{\operatorname{Pr}}\left[\widehat{f}_{S_1}(X) \neq f^{\star}(X) \wedge \widehat{f}_{S_2}(X) \neq f^{\star}(X)\right] \\ 
&\leq 8\max\left\{\frac{d\Log(8e\Log(\min\{en/d,e\tilde{c}\Log^2(1/\delta)/\delta\}))}{n},\frac{\Log(16/\delta)}{n}\right\}\\
&\leq c\max\left\{\frac{d\Log(\Log(\min\{n/d,1/\delta\}))}{n},\frac{\Log(1/\delta)}{n}\right\},
 \end{align*}
where the last inequality holds for $c$ large enough.

We now prove that the event $E_2$ occurs with high probability.
Combining \cref{item:1,item:2} of \cref{helpermarkov} we have, with probability at least $1-\delta/2$ over the randomness of $(S_1,S_2)$, that
\begin{align*}
 \sum_{i = 2}^{\infty}& P\left(R_i\right) \underset{X \sim P_{R_i}}{\operatorname{Pr}}\left[\widehat{f}_{S_1}(X) \neq f^{\star}(X) \wedge \widehat{f}_{S_2}(X) \neq f^{\star}(X)\right]\\
&\leq \sum_{i \not\in I} 0 + 
\sum_{i \in I} \frac{2^ic' d \Log (2^i\Log(1/\delta) / \delta) \Log(1/\delta)}{\delta n}\cdot \frac{5\cdot 2^{-1.1 i} \delta}{\Log^2(1/\delta)} \\
&\leq \frac{5c'd}{n}\sum_{i=2}^{\infty} \frac{  2^{-0.1i}\Log (2^i\Log(1/\delta) / \delta) }{\Log(1/\delta)}\\
&\leq \frac{5c'd}{n}\sum_{i=2}^{\infty} \frac{2^{-0.1i}\cdot \left(i\Log (2)+\Log(\Log(1/\delta))+ \Log(1/\delta)\right) }{\Log(1/\delta)}\\
&\leq  c\frac{ d}{n},
\end{align*}
where the last inequality holds for $c$ large enough.
This concludes the proof.
\end{proof} 
We now move on to prove \cref{condsimon}.
To do so, we will need the following lemma which is a simple consequence of uniform convergence. We defer the proof of this lemma to \cref{conduniformsubsection}.
\begin{restatable}{lem2}{conduniform}\label{conduniform}
Fix a function class $\mc{F} \subseteq \{0,1\}^{\mc{X}}$ with VC dimension $d$. 
Fix a distribution $P$ over $\mc{X}$, target function $f^\star \in \mc{F}$, and $R\subseteq \mathcal{X}$ such that $P(R) \not= 0$.
Fix an ERM algorithm $\widehat{f} : \mc{X} \times \mc{Z}^* \to \{0,1\}$.
For any parameter $\delta\in(0,1/2]$
it holds with probability at least $1-\delta$ over the randomness of $S\sim P^{n}$ that
\begin{align*}
\operatorname{err}_{P_R}\left(\widehat{f}_{S}\right)\leq8\max\left\{\frac{d\Log(eP\left(R\right)n/d)}{P\left(R\right)n},\frac{\Log(4/\delta)}{P\left(R\right)n}\right\}.
\end{align*}
\end{restatable}
We are now ready to prove \cref{condsimon}.
\begin{proof}[Proof of \cref{condsimon}]
We will prove that the event 
\begin{align*}
 E &= E((S_1,S_2))\\
 &\coloneqq \left\{\underset{X \sim P_R}{\operatorname{Pr}}\left[\widehat{f}_{S_1}(X) \neq f^{\star}(X) \wedge \widehat{f}_{S_2}(X) \neq f^{\star}(X)\right] \leq  8\max\left\{\frac{d\Log(8e\Log(eP\left(R\right)n/d))}{P\left(R\right)n},\frac{\Log(8/\delta)}{P\left(R\right)n}\right\} \right\}
\end{align*}
occurs with high probability.
Let $B_1$ denote the (random) set $\{x \in \mc{X} : \widehat{f}_{S_1}(x)\not=f^{\star}(x)\}$ and define the event $E_1 = E_1(S_1) \coloneqq \{ P(R \cap B_1) \not= 0\}$. By the law of total probability, we have 
\begin{equation}\label{eq:split_prob_jointerror}
\Pr_{(S_1, S_2) \sim P^{2n}} \left[ E \right] = \Pr_{(S_1, S_2) \sim P^{2n}} \left[ E_1 \cap E  \right] + \Pr_{(S_1, S_2) \sim P^{2n}} \left[ \bar{E}_1 \cap E \right].
\end{equation}
Furthermore, we can  write the second term on the right hand side of \cref{eq:split_prob_jointerror} as 
\[
\Pr_{(S_1, S_2) \sim P^{2n}} \left[ \bar{E}_1 \cap E \right] = \Pr_{(S_1, S_2) \sim P^{2n}} \left[  E  \mid \bar{E}_1\right]\Pr_{(S_1, S_2) \sim P^{2n}} \left[ \bar{E}_1\right] = \Pr_{(S_1, S_2) \sim P^{2n}} \left[ \bar{E}_1\right].
\]
Combining the identities above, it suffices to show that
\[\Pr_{(S_1, S_2) \sim P^{2n}} \left[  E  \cap {E}_1\right]\geq \Pr_{(S_1, S_2) \sim P^{2n}} \left[ {E}_1\right]-\delta.
\]

Notice that when $E_1$ occurs, then for any measurable set $C\subseteq\mathcal{X}$, the distribution $(P_R)_{B_1}$ (which is the conditional distribution of $P_R$ restricted to $B_1$) satisfies
\[
(P_R)_{B_1}(C) =\frac{P_R\left(C\cap B_1\right)}{P_R\left(B_1 \right)}
=
\frac{P\left(C\cap B_1 \cap R\right)}{P\left(B_1 \cap R\right)}
=P_{R\cap B_1}\left(C\right),
\]
i.e., $(P_R)_{B_1}=P_{R\cap B_1}$.
Thus on $E_1$, the probability that both $\widehat{f}_{S_1}$ and $\widehat{f}_{S_2}$ simultaneously err on a new data point drawn from $P_R$ can be written as
\begin{align}\label{eq8}
\underset{X \sim P_R}{\operatorname{Pr}}\left[\widehat{f}_{S_1}(X) \neq f^{\star}(X) \wedge \widehat{f}_{S_2}(X) \neq f^{\star}(X)\right]= \operatorname{err}_{P_R}\left(\widehat{f}_{S_1}\right)\operatorname{err}_{P_{R\cap B_1}}\left(\widehat{f}_{S_2}\right).
\end{align}

We now bound the right side of \cref{eq8}. To do this, we define the following events over $(S_1,S_2)$:
\begin{align*}
 E_2 = E_2(S_1) \coloneqq\left\{\operatorname{err}_{P_R}\left(\widehat{f}_{S_1}\right)\leq 8\max\left\{\frac{d\Log(eP\left(R\right)n/d)}{P\left(R\right)n},\frac{\Log(8/\delta)}{P\left(R\right)n}\right\} \right\}
 \end{align*}
 and for outcomes of $S_1$ in $E_1$ 
 \begin{align*}
E_3 &= E_3((S_1, S_2)) 
\\
&\coloneqq \left\{
 \operatorname{err}_{P_{R\cap B_1}}\left(\widehat{f}_{S_2}\right)\leq 8\max\left\{\frac{d\Log(eP\left(R\right)n\operatorname{err}_{P_R}\left(\widehat{f}_{S_1}\right)/d)}{P\left(R\right)n\operatorname{err}_{P_R}\left(\widehat{f}_{S_1}\right)},\frac{\Log(8/\delta)}{P\left(R\right)n\operatorname{err}_{P_R}\left(\widehat{f}_{S_1}\right)}\right\} \right\}.\\
\end{align*}

We now show that the event $E_1 \cap E_2 \cap E_3$ happens with probability at least $P[E_1]-\delta$ and that it implies the event $E_1 \cap E$.
Assume that $E_1 \cap E_2 \cap E_3$ occurs. 
If
\begin{align*}
8\max\left\{\frac{d\Log(eP\left(R\right)n/d)}{P\left(R\right)n},\frac{\Log(8/\delta)}{P\left(R\right)n}\right\} =8\frac{d\Log(eP\left(R\right)n/d)}{P\left(R\right)n},
\end{align*} 
we have
\begin{align*}
\operatorname{err}_{P_R}\left(\widehat{f}_{S_1}\right)\operatorname{err}_{P_{R\cap B_1}}\left(\widehat{f}_{S_2}\right) &\le 8\max\left\{\frac{d\Log(eP\left(R\right)n\operatorname{err}_{P_R}\left(\widehat{f}_{S_1}\right)/d)}{P\left(R\right)n},\frac{\Log(8/\delta)}{P\left(R\right)n}\right\}\\
&\le 8\max\left\{\frac{d\Log(8e\Log(eP\left(R\right)n/d))}{P\left(R\right)n},\frac{\Log(8/\delta)}{P\left(R\right)n}\right\}.
\end{align*}
Otherwise, if 
\begin{align*}
8\max\left\{\frac{d\Log(eP\left(R\right)n/d)}{P\left(R\right)n},\frac{\Log(8/\delta)}{P\left(R\right)n}\right\} =8\frac{\Log(8/\delta)}{P\left(R\right)n},
\end{align*} 
we have
\begin{align*}
\operatorname{err}_{P_R}\left(\widehat{f}_{S_1}\right)\operatorname{err}_{P_{R\cap B_1}}\left(\widehat{f}_{S_2}\right) &\le  8\frac{\Log(8/\delta)}{P\left(R\right)n} \cdot 1 .\\
\end{align*}
We can thus conclude that $E_1\cap E_2\cap E_3$ implies $E_1 \cap E$.
Towards showing the bound $\Pr_{(S_1, S_2) \sim P^{2n}} \left[ E_1 \cap E_2 \cap E_3  \right]\ge P[E_1]-\delta$, notice that the bound $\Pr_{(S_1,S_2)\sim P^{2n}}\left[ \bar{E}_2 \right] \leq \delta/2$ can be established from \cref{conduniform} directly.
Furthermore, for \emph{any fixed realization} of $S_1$ such that $P(R \cap B_1) \not= 0$,~\cref{conduniform} implies that 
\[
 \operatorname{err}_{P_{R\cap B_1}}\left(\widehat{f}_{S_2}\right)\leq 8\max\left\{\frac{d\Log(eP\left(R\right)\operatorname{err}_{P_R}\left(\widehat{f}_{S_1}\right)n/d)}{P\left(R\right)\operatorname{err}_{P_R} \left(\widehat{f}_{S_1}\right)n},\frac{\Log(8/\delta)}{P\left(R\right)\operatorname{err}_{P_R}\left(\widehat{f}_{S_1}\right)n}\right\},
\]
with probability at least $1-\delta/2$ over the randomness of $S_2$.
Using the independence of $S_1$ and $S_2$ we have
\begin{align*}
\Pr_{(S_1,S_2)\sim P^{2n}}\left[ E_1 \cap E_2 \cap E_3 \right]  &= \E_{S_1 \sim P^n}\left[ \mathbf{1}_{E_1}\mathbf{1}_{E_2}\Pr_{S_2 \sim P^n}\left[E_3 \right]  \right] \\
&\geq \E_{S_1 \sim P^n}\left[ \mathbf{1}_{E_1}\mathbf{1}_{E_2}\right](1- \delta/2) \\
&\ge (1-\Pr_{S_1 \sim P^n}\left[ \bar{E}_1\right] - \Pr_{S_1 \sim P^n}\left[ \bar{E}_2\right])(1- \delta/2) \\
&\ge (\Pr_{S_1 \sim P^n}\left[ E_1\right] - \delta/2)(1- \delta/2) \\
&\ge \Pr_{S_1 \sim P^n}\left[ E_1\right] -\delta.
\end{align*}
This completes the proof.
\end{proof}

\bibliographystyle{alpha}
\bibliography{refs}
\appendix
\section{Ommited proofs from~\texorpdfstring{\cref{expectationupperboundsection,sec:highprobability}}{Sections 2 and 4}}
In this appendix we prove \cref{lem:expected_uniform_conditional}, \cref{helpermarkov}, and \cref{conduniform}.
These results are by-products of the classic \emph{uniform convergence} result which uniformly bounds the error of any function in $\mathcal{F}$ that is consistent with the training sample. 
To state the result, we first introduce some notation.
For a training sample $S = ((X_1, f^\star(X_1)), \dots, (X_n, f^\star(X_n)))$, let $\mathcal{F}_S$ denote the functions in $\mathcal{F}$ that are consistent with $S$, i.e., $f\in \mathcal{F}_S$ if and only if $f(X_i)=f^\star(X_i)$ for every $i \in [n]$.
\begin{lemma}[Uniform convergence \cite{blumer1989learnability}]\label{lem:uniform}
Fix a function class $\mc{F} \subseteq \{0,1\}^{\mc{X}}$ with VC dimension $d$. 
Fix a distribution $P$ over $\mc{X}$ and target function $f^\star \in \mc{F}$. 
For any parameter $\delta\in(0,1/2]$
it holds with probability at least $1-\delta$ over $S\sim P^n$ that
\[
\sup_{f\in\mathcal{F}_S}\err{f}{P} \le 2\left(\frac{d\log(2en/d) + \log(2/\delta)}{n}\right).
\]
\end{lemma}
In what follows, we will use the slightly weaker bound
\begin{align}\label{uniformconvergenceweakbound}
     \sup_{f\in\mathcal{F}_S} \err{f}{P} \le 4\max\left\{\frac{d\Log(2en/d)}{n}, \frac{\Log(2/\delta)}{n}\right\}.
\end{align}
\subsection{Proof of \texorpdfstring{\cref{lem:expected_uniform_conditional}}{Lg}}\label{lem:expected_uniform_conditionalsubsection}
We now prove \cref{lem:expected_uniform_conditional} which we restate here for convenience.
\conditional*
\begin{proof}
Consider the case that $\Pr_{X \sim P}[X \in R]n\leq 4d$.
In this case the claim follows easily since $\err{\widehat{f}_S}{P_R} \le 1 $ and $x\rightarrow \Log(ex)/x$ is decreasing in $x$ for $x>0$, so $20\frac{d\Log(e\Pr_{X \sim P}[X \in R]n/d)}{\Pr_{X \sim P}[X \in R]n} > 1$.
We now consider the case that $\Pr_{X \sim P}[X \in R]n>4d$.
For any $m \in \mathbb{N}$ we define the event $E_m = E_m(S) = \{ |\{i\in [n]: X_i\in R\}| =  m \}$.
Similarly, we define the event 
\[
E = E(S) =  \bigcup\limits_{m \ge \Pr_{X \sim P}[X \in R]n/2} E_m.
\] 
It follows from a Chernoff bound and our assumption that $\Pr_{X \sim P}[X \in R]n > 4d$ that 
\[
\Pr_{S\sim P^n}[E] \geq 1-\exp\left(-\frac{\Pr_{X \sim P}[X \in R]n}{8}\right) \geq 1-\frac{8}{\Pr_{X \sim P}[X \in R]n}.
\]
Using the law of total probability we have
\begin{equation}\label{eq:conditional_error_main_bound}
\E_{S \sim P^n}\left[\err{\widehat{f}_S}{P_R}\right] \leq  \E_{S \sim P^n}\left[\err{\widehat{f}_S}{P_R} \;\middle|\; E\right]  + \frac{8}{\Pr_{X \sim P}[X \in R] n}.
\end{equation}
So, if we show that for any $m \ge\Pr_{X \sim P}[X \in R]n/2$  that 
\begin{equation}\label{eq:conditional_errorbound_sample_size}
\E_{S \sim P^n}\left[\err{\widehat{f}_S}{P_R} \;\middle|\; E_m \right] \le \frac{12d\Log \left(e\Pr_{X \sim P}[X \in R]n/d\right)}{\Pr_{X \sim P}[X \in R]n},
\end{equation}
the claim follows from one more application of the law of total probability applied to the first term on the right hand side of \cref{eq:conditional_error_main_bound}.

We now prove \cref{eq:conditional_errorbound_sample_size}.
Using the non-negativity of $\err{\widehat{f}_S}{P_R}$ we have
\begin{align}\label{eq:expectationlemma cond 3}
    \E_{S \sim P^n}\left[\err{\widehat{f}_S}{P_R} \;\middle|\; E_m\right] &=\int_{0}^{\infty} \Pr_{S\sim P^n}\left[\err{\widehat{f}_S}{P_R}>x\;\middle|\; E_m\right]  dx\nonumber\\ &\leq \frac{4d\Log \left(2em/d\right)}{m}+\int_{\frac{4d\Log \left(2em/d\right)}{m}}^{1} \Pr_{S\sim P^n}\left[\err{\widehat{f}_S}{P_R}>x\;\middle|\; E_m\right]  dx.
\end{align}
Notice that conditioned on $E_m$, the $m$ samples that land in $R$ form an i.i.d.\ sample from the conditional distribution $P_R$. 
Thus, any ERM trained on $S$ is also consistent with $m$ i.i.d.\ samples from $P_R$, so we can apply uniform convergence (\cref{lem:uniform}) to control the error of any ERM when measured with respect to the conditional distribution $P_R$.
Setting $\delta = 2^{1-\frac{mx}{4}}$ we have 
\begin{align*}
\int_{\frac{4d \Log \left(2em/d\right)}{m}}^{1} \Pr_{S\sim P^n}\left[\err{\widehat{f}_S}{P_R}>x\;\middle|\; E_m\right]  dx &= \int_{\frac{4d \Log \left(2em/d\right)}{m}}^{1} \Pr_{S\sim P^n}\left[\err{\widehat{f}_S}{P_R}> \frac{4\Log(2/\delta)}{m}\;\middle|\; E_m\right]  dx\\
&\le 2 \int_{\frac{4d \Log \left(2em/d\right)}{m}}^{1} 2^{-\frac{mx}{4}}  dx\\
&\le 2\left(\frac{4 \cdot 2^{-d\Log \left(2em/d\right)}}{m \ln(2)}\right) \le \frac{2d\Log \left(2em/d\right)}{m}.
\end{align*}
Here, the first equality follows from  the fact that $m \geq \Pr_{X \sim P}[X \in R]n/2 \ge 2d$ and our choice of $\delta$.
The second inequality follows from \cref{uniformconvergenceweakbound} and the final inequality follows from the fact that $d\Log(2em/d) \ge 2$.
Now, using the fact that $x \rightarrow \Log(2ex)/x$ is decreasing for $x > 0$ together with the fact that $m \geq \Pr_{X \sim P}[X \in R]n/2 \geq 2d$, we conclude 
\[
\E_{S \sim P^n}\left[\err{\widehat{f}_S}{P_R} \;\middle|\; E_m\right] \le \frac{6d\Log \left(2em/d\right)}{m} < \frac{12d\Log \left(e\Pr_{X \sim P}[X \in R]n/d\right)}{\Pr_{X \sim P}[X \in R]n},
\]
as claimed.
\end{proof}
\subsection{Proof of \texorpdfstring{\cref{helpermarkov}}{Lg}}\label{app:helpermarkov}
We now prove~\cref{helpermarkov} which we restate here for convenience. 
\helpermarkov*
\begin{proof}
We first prove \cref{item:1}.
Towards a contradiction, assume that there is an $i \in \mathbb{N}$ such that 
\[
\frac{P\left(R_i\right)n}{d} \geq \frac{c' 2^{i}\Log (2^i\Log(1/\delta)/\delta) \Log(1/\delta)}{\delta}
\]
for a constant $c'$ that we will choose below.
By changing the order of expectations we have
\[
\underset{S_1 \sim P^n}{\mathbb{E}}\left[\operatorname{err}_{P_{R_i}}\left(\widehat{f}_{S_1}\right)\right]=\underset{X \sim P_{R_i}}{\mathbb{E}}\left[\underset{S_1 \sim P^n}{\mathbf{P r}}\left[\widehat{f}_{S_1}(X) \neq f^{\star}(X)\right]\right]=\underset{X \sim P_{R_i}}{\mathbb{E}}\left[p_X\right].
\]
Using the above together with the fact that $p_X \geq 2^{-i} \delta/\Log(1/\delta)$ for any $X\in R_i$, we can conclude that
\begin{align}\label{eq1}
\underset{S_1 \sim P^n}{\mathbb{E}}\left[\operatorname{err}_{P_{R_i}}\left(\widehat{f}_{S_1}\right)\right] > \frac{\delta}{2^{i}\Log(1/\delta)}.
\end{align}
On the other hand, using \cref{lem:expected_uniform_conditional} we conclude that there is a universal constant $\hat{c}$ such that
\begin{align*}
\underset{S_1 \sim P^n}{\mathbb{E}}\left[\operatorname{err}_{P_{R_i}}\left(\widehat{f}_{S_1}\right)\right] \leq \hat{c} \frac{d \Log \left(P\left(R_i\right) n/ d\right)}{P\left(R_i\right) n}.
\end{align*}
Combining this inequality with the fact that $\Log(x)/x$ is a decreasing function for $x>0$, we have
\begin{align*}
&\underset{S_1 \sim P^n}{\mathbb{E}}\left[\operatorname{err}_{P_{R_i}}\left(\widehat{f}_{S_1}\right)\right] \leq \hat{c} \frac{\delta\Log\left(\frac{c' 2^{i}\Log (2^i\Log(1/\delta)/\delta) \Log(1/\delta)}{\delta} \right)} {c' 2^{i}\Log (2^i\Log(1/\delta)/\delta) \Log(1/\delta)}\\
&=\hat{c}\frac{2^{-i}\delta}{ \Log(1/\delta)}\cdot \frac{\Log\left(\frac{c' 2^{i}\Log (2^i\Log(1/\delta)/\delta) \Log(1/\delta)}{\delta} \right)}{c'\Log \left(\frac{2^i\Log(1/\delta)}{\delta}\right)}\\
&\le \hat{c}\frac{2^{-i}\delta}{ \Log(1/\delta)}\cdot \frac{\Log\left(c'\right)+\Log\left(\Log \left(\frac{2^i\Log(1/\delta)}{\delta}\right) \right)+\Log\left(2^i\frac{ \Log(1/\delta)}{\delta} \right)}{c'\Log \left(\frac{2^i\Log(1/\delta)}{\delta}\right)}.
\end{align*}
However, for $c'$ large enough, the above is less than $2^{-i}\delta/\Log(1/\delta)$ which contradicts the lower bound \cref{eq1}.
Thus, we have shown that there is a constant $c'$ such that
\begin{align*}
\frac{P\left(R_i\right)n}{d}\leq \frac{c' 2^{i}  \Log (2^i\Log(1/\delta)/\delta) \Log(1/\delta)}{\delta}, 
\end{align*}
which proves \cref{item:1}.

We now prove \cref{item:2}.
We will show that for each $i \in I $ with probability at least $1-2^{-0.9i+2}\delta/5$ we have 
\begin{equation}\label{eq:simple_markov}
\underset{X \sim P_{R_i}}{\operatorname{Pr}}\left[\widehat{f}_{S_1}(X) \neq f^{\star}(X) \wedge \widehat{f}_{S_2}(X) \neq f^{\star}(X)\right] \leq \frac{5\cdot 2^{-1.1 i} \delta}{\Log^2(1/\delta)} .
\end{equation}
Applying a union bound implies that the above holds simultaneously for every $i \in I$ with probability at least $1-\sum_{i\geq 2}2^{-0.9i+2}\delta/5\geq 1-\delta/2$.
To see that \cref{eq:simple_markov} holds for each $i \in I$, notice that we can use the fact that $S_1$ and $S_2$ are i.i.d.\ samples together with the fact that $p_X\leq 2^{-i+1}\delta /\Log(1/\delta)$ for $X\in R_i$ to conclude that
\[
\underset{X \sim P_{R_i}}{\mathbb{E}}\left[\underset{S_1, S_2 \sim P^n}{\operatorname{Pr}}\left[\widehat{f}_{S_1}(X) \neq f^{\star}(X) \wedge \widehat{f}_{S_2}(X) \neq f^{\star}(X)\right]\right]=\underset{X \sim P_{R_i}}{\mathbb{E}} p_X^2 \leq \frac{2^{-2 i+2}\delta^2}{\Log^2(1/\delta)} .
\]
Combining this with Markov's inequality we have
\begin{align*}&
\underset{S_1, S_2 \sim P^n}{\operatorname{Pr}}\left[\operatorname{Pr}_{X \sim P_{R_i}}\left[\widehat{f}_{S_1}(X) \neq f^{\star}(X) \wedge \widehat{f}_{S_2}(X) \neq f^{\star}(X)\right] > \frac{5\cdot 2^{-1.1 i} \delta}{\Log^2(1/\delta)}\right] \\ &\leq \frac{2^{-2 i+2}\delta^2}{\Log^2(1/\delta)} \frac{\Log^2(1/\delta)}{5\cdot 2^{-1.1 i} \delta}= 2^{-0.9i+2}\delta/5,
\end{align*}
which proves the claim.    
\end{proof}

\subsection{Proof of \texorpdfstring{\cref{conduniform}}{Lg}}\label{conduniformsubsection}
We now prove \cref{conduniform} which we restate here for convenience.
\conduniform*
\begin{proof}
If $8\Log(4/\delta)/(P(R)n)\geq 1$ we are done as $\err{f}{P_R}\leq 1$. Thus, for the remainder of the proof we will assume that $8\Log(4/\delta)/(P(R)n) < 1$, which is equivalent to $P(R)n \geq 8\Log(4/\delta)$.
Define the event
\[
E = E(S) =  \{ |\{i\in [n]: X_i\in R\}| \ge P(R)n/2\}.
\] 
Using the Chernoff bound and our assumption that $P(R)n \geq 8\Log(4/\delta)$, we have
\[
\Pr_{S \sim P^n}\left[E\right] \ge 1 - \exp\left(-\frac{P(R)n}{8} \right) \ge 1 - \delta/2.
\]
Notice that conditioned on $E$, the $M \ge P(R)n/2 \ge 1$ samples\footnote{The number of samples $M$ is random.} that land in $R$ form an i.i.d.\ sample from the conditional distribution $P_R$. 
Thus when $E$ occurs, any ERM trained on $S$ is also consistent with $M \ge P(R)n/2 \ge 1$ i.i.d.\ samples from $P_R$, so \cref{lem:uniform} yields, with probability at least $1-\delta/2$ over the randomness of $S$, that
\begin{align*}
\err{\widehat{f}_S}{P_R} &\le 4\max\left\{\frac{d\Log(2eM/d)}{M}, \frac{\Log(4/\delta)}{M}\right\}\\
&\le 8\max\left\{\frac{d\Log(eP(R)n/d)}{P(R)n}, \frac{\Log(4/\delta)}{P(R)n}\right\}.
\end{align*}
Here, the second inequality follows from the fact that $x \rightarrow \Log(2ex)/x$ is decreasing for $x > 0$.
Using the law of total probability we get that
\begin{align*}
&\Pr_{S \sim P^n}\left[\err{\widehat{f}_S}{P_R}  > 8\max\left\{\frac{d\Log(eP(R)n/d)}{P(R)n}, \frac{\Log(4/\delta)}{P(R)n}\right\} \right] \\
&\le \Pr_{S \sim P^n}\left[\bar{E}\right] + \Pr_{S \sim P^n}\left[\err{\widehat{f}_S}{P_R}  > 8\max\left\{\frac{d\Log(eP(R)n/d)}{P(R)n}, \frac{\Log(4/\delta)}{P(R)n}\right\} \;\middle|\; E \right]\\
&\le \delta/2 + \delta/2 = \delta.
\end{align*}
This concludes the proof.
\end{proof}\label{appendixA}
\section{Ommited proofs from~\texorpdfstring{\cref{lowerboundv3}}{Section 3}}\label{app:proof_sum_geometric}
In this appendix we prove~\cref{eq:sum_geometric}. 
We will show that
\[
\Pr_{Y \sim Q}\left[Y \ge m\right] = \Pr_{Y_t \sim Q_t}\left[\sum_{t=0}^{m_2d -d -1} Y_t  \ge m\right] \ge \sqrt{\frac{2}{3}}.
\]
Let $p^\star =p_{m_2 d - d - 1}= \frac{d+1}{m_1 m_2}$ be the smallest parameter $p_t$ of the geometric random variables $\{Y_t\}_{t=0}^{m_2 d -d -1}$ that we consider. 
We make use of the following well known concentration inequality for sums of geometric random variables:
\begin{equation}
\Pr_{Y\sim Q}\left[Y\leq \lambda \E_{Y\sim Q}\left[Y\right] \right]\leq \exp\left(-p^\star\E_{Y\sim Q}[Y](\lambda -1-\ln(\lambda))\right),\label{eq:concentration}
\end{equation}
which holds for any $0 < \lambda \le 1$ (see \cite[Theorem 3.1]{JANSON20181}).
Let $\lambda = 1/4$.
We will show $\E_{Y\sim Q}\left[Y\right]\geq 4m$ and $p^{\star}\E_{Y\sim Q}\left[Y\right]\geq 4$ which combined with $1/4-1-\ln\left(1/4\right)\geq 1/2$ and \cref{eq:concentration} gives us
\begin{align*}
\Pr_{Y\sim Q}\left[Y\leq m  \right]\leq \Pr_{Y\sim Q}\left[Y\leq \E_{Y\sim Q}[Y]/4  \right]\leq \exp\left(-p^{\star}\E_{Y\sim Q}\left[Y\right]/2\right)\leq \exp(-2) \le 1- \sqrt{2/3}.
\end{align*}
This implies  $\Pr_{Y \sim Q}\left[Y \ge m\right]\geq \sqrt{2/3}$ as required.
We first show that $\E_{Y\sim Q}[Y]\geq 4m$.
We have
\begin{align}\label{eq:coup12}
\E_{Y\sim Q}[Y]=\sum_{t=0}^{m_2d-d-1} \frac{m_1m_2}{m_2d-t}= \sum_{i=d+1}^{m_2d} \frac{m_1m_2}{i} \geq m_1m_2\ln\left(\frac{m_2}{2}\right).
\end{align}
Plugging in the definition of $m_1$ and $m_2$ into \cref{eq:coup12} and using the fact that $\lceil x \rceil \le 2x$ for $x \ge 0.5$ gives us 
\[
\E_{Y\sim Q}[Y]\geq \frac{8\mega m}{\ln \left(\ln \left(2\mega m/d\right)\right)} \ln \left(\frac{\ln (2\mega m/d)}{\ln \left(\ln \left(2\mega m/d\right)\right)}\right).
\]
For $\mega$ large enough we have $\ln(\ln(2 \mega m/d)) > 0$, $\frac{\ln(2\mega m/d)}{\ln(\ln(2\mega m/d))} \ge \sqrt{\ln(2\mega m/d)}$ and $C > 1$, so
\[
\E_{Y\sim Q}[Y]\geq \frac{8\mega m}{\ln \left(\ln \left(2\mega m/d\right)\right)} \ln \left(\sqrt{\ln(2\mega m/d)}\right) \ge 4 m.
\]

We now show that $p^\star \E_{Y\sim Q}[Y] \ge 4$.
Using the fact that $p^\star \ge 2/(m_1m_2)$ together with \cref{eq:coup12} gives us
\[
p^\star \E_{Y\sim Q}[Y] \ge 2\ln\left(\frac{m_2}{2}\right) \ge 2\ln \left(\frac{\ln (2\mega  m/d)}{\ln \left(\ln \left(2\mega  m/d\right)\right)}\right) \ge 4,
\]
where the last inequality holds for $\mega$ large enough.
\label{appendixB}
\end{document}